\documentclass[conference]{IEEEtran}
 \IEEEoverridecommandlockouts

\usepackage{microtype}
\usepackage{graphicx}
\usepackage{subfigure}
\usepackage{booktabs} 
\usepackage{hyperref}
\usepackage{url}            
 
\usepackage{amsfonts}       
\usepackage{nicefrac}       

\usepackage{graphicx} 
\graphicspath{{figures/}}
\usepackage{epsfig} 
\usepackage{ragged2e}
\usepackage{amsmath}
\usepackage{amsthm}
\usepackage{amssymb}
\usepackage{epstopdf}
\usepackage{algorithm}
\usepackage{algorithmic}
\usepackage{multirow}
\usepackage{rotating}
\usepackage{subfigure}
\usepackage{color}
\usepackage{mysymbol}
\usepackage{url}

\newtheorem{theorem}{Theorem}[section]

\newtheorem{prop}[theorem]{Proposition}

\newtheorem{pr}{Problem}
\newtheorem{rem}[theorem]{Remark}

\def\BibTeX{{\rm B\kern-.05em{\sc i\kern-.025em b}\kern-.08em
    T\kern-.1667em\lower.7ex\hbox{E}\kern-.125emX}}
\begin{document}

\title{Detection of Adversarial Physical Attacks in Time-Series Image Data\\
\thanks{$^*$Authors with equal contribution.}
}

\makeatletter
\newcommand{\linebreakand}{%
  \end{@IEEEauthorhalign}
  \hfill\mbox{}\par
  \mbox{}\hfill\begin{@IEEEauthorhalign}
}
\makeatother

 \author{
 \IEEEauthorblockN{Ramneet Kaur$^*$}
 \IEEEauthorblockA{\textit{Computer and Information Science} \\
 \textit{University of Pennsylvania}\\
 Philadelphia, PA, USA \\
 ramneetk@seas.upenn.edu}
 \and
 \IEEEauthorblockN{Yiannis Kantaros$^*$}
 \IEEEauthorblockA{\textit{Electrical and Systems Engineering} \\
 \textit{Washington University in St. Louis}\\
 St. Louis, MO, USA \\
 ioannisk@wustl.edu}
  \and
 \IEEEauthorblockN{Wenwen Si$^*$}
 \IEEEauthorblockA{\textit{Computer and Information Science} \\
 \textit{University of Pennsylvania}\\
 Philadelphia, PA, USA \\
 wenwens@seas.upenn.edu}
 \linebreakand
 \IEEEauthorblockN{James Weimer}
 \IEEEauthorblockA{\textit{Computer Science} \\
 \textit{Vanderbilt University}\\
 Nashvilee, TN, USA \\
 james.weimer@vanderbilt.edu}
\and
 \IEEEauthorblockN{Insup Lee}
 \IEEEauthorblockA{\textit{Computer and Information Science} \\
 \textit{University of Pennsylvania}\\
 Philadelphia, PA, USA \\
 lee@cis.upenn.edu}
 }

\maketitle

\begin{abstract}
Deep neural networks (DNN) have become a common sensing modality in autonomous systems as they allow for semantically perceiving the ambient environment given input images. Nevertheless, DNN models have proven to be vulnerable to adversarial digital and physical attacks. To mitigate this issue, several detection frameworks have been proposed to detect whether a \textit{single} input image has been manipulated by adversarial \textit{digital} noise or not. In our prior work, we proposed a  real-time detector, called VisionGuard (VG), for adversarial physical attacks against single input images to DNN models. Building upon that work, we propose $\text{VisionGuard}^*$ ($\text{VG}^*$), that couples VG with majority-vote methods, to detect adversarial physical attacks in \textit{time-series image data}, e.g., videos. This is motivated by autonomous systems applications where images are collected over time using onboard sensors for decision making purposes. We emphasize that majority-vote mechanisms are quite common in autonomous system applications (among many other applications), as e.g., in autonomous driving stacks for object detection. In this paper, we investigate, both theoretically and experimentally, how this widely used mechanism can be leveraged to enhance performance of adversarial detectors. We have evaluated $\text{VG}^*$ on videos of both clean and physically attacked traffic signs generated by a state-of-the-art robust physical attack. We provide extensive comparative experiments against detectors that have been designed originally for out-of-distribution data and digitally attacked images.
\end{abstract}


\section{Introduction}\label{sec:introduction}

Deep neural networks (DNN) have seen renewed interest in the last decade due to the vast amount of available data and recent advances in computing. In autonomous systems, DNN are typically used as feedback controllers, motion planners, or perception modules such as object detection and classification. 
Nevertheless, brittleness of DNN has resulted in unreliable behavior and public failures, e.g., Tesla cars' crashes and Uber running a red light \cite{uber_crash}.  
In fact, several adversarial attack methods have been proposed recently that aim to minimally manipulate inputs to DNN models to cause desired incorrect outputs that would benefit an attacker. In image classification, that is also considered in this paper, adversarial attacks are typically categorized as digital and physical ones. The latter augment the physical world by placing adversarial stickers on the surface of objects of interest  \cite{karmon2018lavan,li2019adversarial,eykholt2018robust,brown2017adversarial,boloor2020attacking} (see e.g., Fig. \ref{fig:X}), while the former embeds an imperceptible amount of digital noise to the classifier input data \cite{carlini2017adversarial,papernot2016limitations,choi2022argan}. 
Digital attacks typically assume a threat model in which the adversary can feed data directly into the machine learning component. This is not necessarily the case for systems operating in the physical world, such as mobile robots or driverless cars that are using signals from cameras and other sensors as input. Motivated by these applications, this paper focuses exclusively on adversarial physical attacks.

Several \textit{detectors} have been proposed recently to detect adversarially attacked images; see Section \ref{sec:related}. Nevertheless, the large majority of these works primarily focus on determining whether a single image has been manipulated adversarially by a digital attack or not. To the contrary, in this paper, we focus on \textit{detecting} adversarial \textit{physical} attacks against DNN for image classification tasks in \textit{time-series image data} (e.g., videos).
Our work is motivated by autonomous driving applications, where a time-ordered sequence of images, collected by onboard cameras, containing  objects of interest is available that is used for control purposes. For instance, consider an autonomous car collecting images of a stop sign as it approaches an intersection. Our goal is to determine if the objects (e.g., traffic signs) in the collected images have been manipulated by adversarial targeted physical attacks proposed in \cite{eykholt2018robust}; examples of attacked traffic signs using \cite{eykholt2018robust} are shown in Fig. \ref{fig:X}.
In particular, we consider the following problem. We assume that a time-ordered sequence of images containing a single object of interest is given. We also assume the existence of a perfect object tracking algorithm that is capable of creating an ordered sequence containing the same object. This assumption can be relaxed by using e.g., recent probabilistic data association algorithms that can associate semantic sensor measurements to objects; see e.g., \cite{bowman2017probabilistic}. 
Our goal is to determine whether the object captured in the time-series data has been adversarially manipulated \cite{eykholt2018robust}. To achieve this, we design a new, simple yet effective, time-series detector, called $\text{VisionGuard}^*$ ($\text{VG}^*$). The proposed time-series detector builds upon (i) VisionGuard (VG), a recently proposed single-image detector against various digital attacks and the physical attacks \cite{kantaros2021real}, and (ii) majority vote mechanisms. Specifically, $\text{VG}^*$  consists of the following two steps: (1) VG is applied to all images that lie within a sliding detection window (capturing how many past images will be used) classifying each image as either adversarial or clean; (2) majority-vote is applied over the VG outputs. If the majority of the VG outputs is `adversarial', then  $\text{VG}^*$ classifies the object under investigation as adversarial; otherwise, it is considered clean. 
%
%
Intuitively, the detection performance of  $\text{VG}^*$ should increase as the size of available data increases or, equivalently, the length of the sliding window increases. To formalize this intuition, we provide conditions about the length of the sliding window and the detection accuracy of VG, under which the time-series detector achieves better detection performance than the single-image detector. In other words, this result shows that, as expected, exploiting past observations of an object interest - instead of relying only on the current input image - can enhance performance of adversarial detectors resulting in safer perception-based control loops.
We emphasize that majority-vote mechanisms are quite common in autonomous system applications (among several other applications), as e.g., in autonomous driving stacks \cite{Apollo} for object detection. Our goal is to investigate how this widely used mechanism can be leveraged to enhance performance of adversarial detectors. 

We demonstrate the efficiency of $\text{VG}^*$ on videos of clean and physically attacked traffic signs using stickers designed as in \cite{eykholt2018robust}. 
Additionally, we highlight that VG can be replaced by any other single-image detector to construct a time-series detector. Particularly, in our experiments, we compare the performance of $\text{VG}^*$ against time-series detectors constructed by simply replacing VG in step (1) with the detectors developed in \cite{lee2018simple,liang2018detecting}. 
We note that \cite{lee2018simple} proposes a detector for digitally attacked and out-of-distribution (OOD) data, while \cite{liang2018detecting} focuses exclusively on digital attacks.
Our  experiments validate the previously theoretical connection between the length of detection window and the accuracy of the employed single-image detector, i.e., the more accurate the single-image detector, the better the performance of the corresponding time-series detector as the window length increases. 
Finally, in the appendix, we provide new comparative experiments of VG against the detectors in \cite{lee2018simple,liang2018detecting} on adversarial patches \cite{brown2017adversarial} that do not exist in the original paper that proposed VG \cite{kantaros2021real}. 
The adversarial attack in \cite{brown2017adversarial} generates patches that can be placed 
in the vicinity of an object of interest to cause desired misclassifications. 
Our comparative experiments further validate the claim made in \cite{kantaros2021real} that existing detectors that have been originally developed for digitally attacked or OOD data, as e.g.,  \cite{lee2018simple,liang2018detecting}, may perform poorly at detecting physical attacks.

\begin{figure}[t]
  \centering
  \subfigure{
    \label{fig:stop2}
    \includegraphics[width=0.30\linewidth]{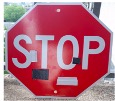}}
     \subfigure{
    \label{fig:yield1}
    \includegraphics[width=0.30\linewidth]{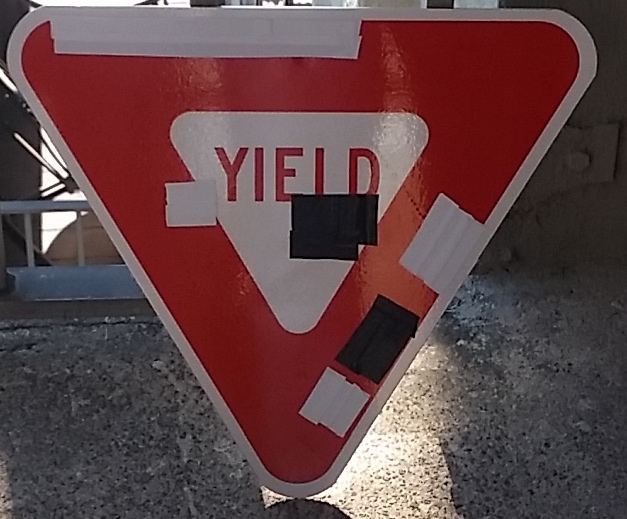}}
       \subfigure{
    \label{fig:speed2}
    \includegraphics[width=0.23\linewidth]{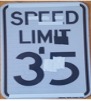}}
\vspace{-0.35cm}
  \caption{
   Traffic signs manipulated adversarially by the $\text{RP}_2$ attack \cite{eykholt2018robust} to fool LISA-CNN. The stickers have been placed so that the stop, yield and speed limit signs are classified as
   speed limit 35, speed limit 35, and turn right signs, respectively \cite{kantaros2021real}.\vspace{-0.6cm}
  }
  \label{fig:X}
\end{figure}


\subsection{Related Work}\label{sec:related}
\vspace{-0.1cm}
To protect DNN-based image classifiers from adversarial attacks, several detection and monitoring methods have recently been proposed. Similar to our work, these methods aim to identify whether a given input image has been manipulated or not. As it will discussed later in the text, unlike our work, the large majority of existing detectors focuses on \textit{single} images that may have been under \textit{digital} attacks. A summary of such works follows.
To detect adversarial images, application of image transformation and processing techniques to input images has been proposed in \cite{tian2018detecting,meng2017magnet,liang2018detecting,liu2021feature,idecode}. VG, mentioned earlier, also lies in this category \cite{kantaros2021real}. 
Adversarial detectors that rely on identifying adversarial subspaces in a manifold of clean data have also been proposed \cite{feinman2017detecting,jha2018detecting,crecchi2019detecting,pang2018towards, ma2018characterizing,lee2018simple}. 
Related is also the recent work in \cite{rossolini2022increasing} that uses coverage criteria to detect adversarial inputs. 
Defenses that rely on training separate DNN that are capable of identifying adversarial images have been proposed as well \cite{fidel2020explainability,cohen2020detecting,liu2022towards}. 
%
Ensembles of previously discussed detectors and defenses against adversarial attacks is considered in \cite{he2017adversarial}.
A more comprehensive summary of existing defenses against adversarial inputs can be found in \cite{deng2021deep,bulusu2020anomalous}.
We emphasize that the above defenses have  been evaluated only on digital attacks and on single input images. 

OOD detection techniques proposed in \cite{cai2020real,yang2022interpretable} have been evaluated against a physical attack \cite{boloor2020attacking}.  Reconstruction error by variational auto-encoder on a \textit{single input image} relative to the training data is used as the non-conformity measure (NCM) for OOD detection in~\cite{cai2020real}.~\cite{yang2022interpretable} propose computing prototypes from the training data and using the distance of an input image from these prototypes for OOD detection on a \textit{single input image}.
A similar line of work for OOD detection in time-series data based on NCM has been proposed~\cite{codit}. This work uses error in equivariance with respect to temporal transformations learned on the in-distribution data as the non-conformity score for detection. They, however, do not consider the detection of adversarial attacks.
%
%

Defenses against adversarial patches have also been proposed in \cite{chiang2020certified,xiang2021patchguard,zhang2020clipped,sridhar2022towards}, where robust training methods have been proposed to enhance the resiliency of the DNN classifier against adversarial inputs.  
Unlike these, this paper focuses on \textit{detecting} physical attacks without modifying the structure of the DNN image classifier. In fact, detection frameworks are complementary to robust training methods, as the former essentially aim to monitor the input to DNN models (whether they are trained in a robust fashion or not).  Moreover, as discussed earlier, we note that the proposed time-series detector is modular in the sense that VG (step (1)) can be replaced by any other single-image detector such as \cite{boloor2020attacking} or \cite{lee2018simple,liang2018detecting}. 
%
Additionally, we provide the first evaluation of adversarial detectors \cite{lee2018simple,liang2018detecting}, originally proposed for digital attacks and/or OOD data, against state-of-the art robust physical attacks \cite{eykholt2018robust,brown2017adversarial} and we show that under certain conditions their performance can be improved by exploiting historical data. 

\subsection{Contribution}
The contributions of this paper can be summarized as follows. 
\textit{First}, we introduce a new, simple, yet effective and modular, detection framework for physical attacks in time-series image data.
\textit{Second}, unlike the large majority of existing detectors that focus on digital attacks or OOD data, this paper focuses on physical attacks and provides the first evaluation of these detectors against physical attacks. \textit{Third}, we show that the proposed framework is computationally efficient allowing for its employment in real-time applications, where images are collected online. \textit{Fourth}, we provide conditions for VG (or any other employed single-image detector) and the length of the sliding detection window (or equivalently the size of historical data) that need to be satisfied so that the time-series detector $\text{VG}^*$ performs better than VG (or the corresponding employed single-image detector). \textit{Fifth}, we provide extensive comparative experiments against robust physical attacks demonstrating the efficiency of $\text{VG}^*$. 
%
%

\section{Problem Formulation}\label{sec:attacks}

%
Consider a classifier $f:\ccalX\rightarrow \ccalC$, where $\ccalX$ is the set of images $x\in\mathbb{R}^{m}$, where $m$ is the number of pixels, and $\ccalC$ is the set of labels. Let $L(x)$ and $f(x)$ denote the true and the predicted label of image $x\in\ccalX$, respectively. 
Consider also an attacker that aims to attack $f$ to cause desired misclassifications. In particular, here we consider physical attacks that aim to place adversarial stickers on the surface of objects of interest (e.g., traffic signs) so that an image $x$ that includes a physically attacked object gets a desired (wrong) label $t(x)$ i.e., $f(x)=t(x)$ and $t(x)\neq L(x)$. In section \ref{sec:RP2}, we review a targeted robust physical attack \cite{eykholt2018robust}.  In Section \ref{sec:detectPr}, we present the adversarial detection problem that this paper addresses.

%
%

\vspace{-0.2cm}
\subsection{Robust Physical Perturbations}\label{sec:RP2}
A targeted robust physical attack is presented in \cite{eykholt2018robust} to generate robust visual adversarial perturbations under different physical conditions. The first step requires to solve the following optimization problem that generates adversarial (digital) noise $\delta$: $\underset{\delta}{\text{min}}~\lambda||M_x \delta||_p+NPS+\mathbb{E}_{x_i\sim X^V}J(f(x_i+T_i(M_x\delta)),y^*),$
where (i) $M_x$ is a mask applied to image $x$ to ensure that the perturbation is applied only to the surface of the object of interest (e.g., on a traffic sign and not in the background); (ii) NPS is a non-printability score to account for fabrication error; (iii) $X^V$ refers to a distribution of images containing an object of interest (e.g., a stop sign) under various environmental conditions captured by digital and physical transformations (resulting in attacks being robust to various environmental conditions) ; (iv)  $T_i(\cdot)$ denotes the alignment function that maps transformations on the object to transformations on
the perturbation (e.g., if the object is rotated, the perturbation is rotated as well); and (v) $y^*$ is the target label. 
Finally, an attacker will print out the optimization result on paper, cut out the perturbation $M_x$, and put it onto the target object. The perturbation generated using the $L_1$ norm along with its physical application is shown in Figure \ref{fig:attackedl1}.  Observe in this figure, that the $L_1$ norm generates a sparse attack vector allowing the attacker to physically implement the attack with black/white stickers, as shown in Fig. \ref{fig:X}. 

\begin{figure}[t]
  \centering
    \includegraphics[width=1\linewidth]{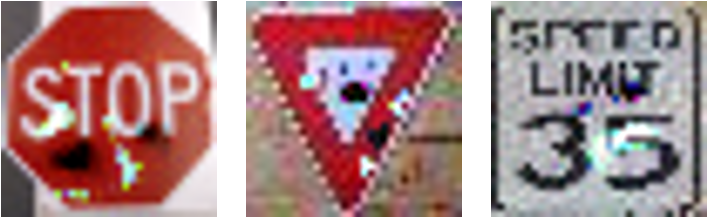}\vspace{-0.65cm}
  \caption{Perturbation generated by the $\text{RP}_2$ attack using the authors-provided code under the $\ell_1$ norm so that a stop, yield, and speed limit 35 sign are misclassified by the LISA CNN \cite{eykholt2018robust} as speed limit 35, speed limit 35, and turn right sign, respectively. These images are used as a guide to place adversarial stickers on the surface of traffic signs as in Fig. \ref{fig:X}.}\vspace{-0.4cm}
  \label{fig:attackedl1}
\end{figure}
\subsection{Detection Problem}\label{sec:detectPr}
\vspace{-0.1cm}
Consider an input image $x_t$, collected at a time instant $t\geq0$, to the DNN image classifier $f$ which contains a single object of interest. Consider also an ordered sequence $X_{0:t-1}=x_0,x_1,\dots,x_{t-1}$ of images $x_n$, $n\in\{0,1,\dots,t-1\}$ containing the same object of interest as $x_t$ does. For instance, consider an autonomous car approaching a stop sign. In this case, $x_t$ is the current input image - containing the stop sign  - to the DNN-based perception system $f$ while $X_{0:t-1}$ denotes a time-ordered sequence of observations/images of the same stop sign taken as the car was approaching it. Our goal in this paper is to build a detection method that determines if $x_t$ is an adversarial image, i.e., if $x_t$ contains an object that has been physically attacked, by leveraging the available historical data $X_{0:t-1}$. We note that in this paper we assume the existence of a perfect object tracking algorithm that is capable of creating an ordered sequence $X_{0:t}$ containing the same object. This assumption can be relaxed by using e.g., recent data association algorithms \cite{bowman2017probabilistic}, which however, is out of the scope of this paper. 
The problem that this paper addresses can be summarized as follows:



\begin{pr}\label{pr1}
Given (i) a DNN $f:\ccalX\rightarrow\ccalC$ that is subject to physical attacks discussed in Section \ref{sec:RP2}; (ii) an input image $x_t\in\ccalX$ under investigation containing an object of interest; and (iii) an ordered sequence of past observations/images $X_{0:t-1}$ containing the same object of interest as $x_t$ does, our goal is to design a detector $f_d:\ccalX_{0:t}\rightarrow\{\text{adversarial, clean}\}$ that classifies the current input image $x_t$ as adversarial or clean given a sequence of past observations $\ccalX_{0:t-1}$.
\end{pr}


\section{$\text{VisionGuard}^*$: A New Adversarial Time-Series Detector }\label{sec:ImageDetector}

To solve Problem \ref{pr1}, we present a time-series detector that builds upon VisionGuard (VG), an existing single-image adversarial detector \cite{kantaros2021real}. To this end, first, in Section \ref{sec:vg}, we review VG. Second, in Section \ref{sec:timeSeriesDet}, we propose a time-series detector to address Problem \ref{pr1}. Finally, in Section \ref{sec:Prop}, we demonstrate the effect of using historical data, that exist in time-series data, on detection performance.
\vspace{-0.1cm}

\subsection{VisionGuard: A Single-Image Adversarial Detector}\label{sec:vg}


VG is a dataset- and attack-agnostic detector that relies on the observation that adversarial inputs are not robust to certain transformations in the sense that these transformations change the output of the DNN significantly \cite{kantaros2021real}. Typically, such transformations are (i) label-invariant, i.e., the true class of the object of interest should not change under this transformation, and (ii) squeeze out features that may be unnecessary for correct classification such as adversarial components. 
In \cite{kantaros2021real}, it was shown that lossy compression and brightness transformations can alleviate the effect of digital and physical attacks, respectively.

VG comprises four steps to detect whether an input image $x$ is adversarial or not. 
First, the input image $x$ is fed to the classifier $f$ to get the softmax output denoted by $\bbg(x)$. Second, a user-specified transformation $t:\ccalX\rightarrow\ccalX$ is applied to $x$ to get an image $x'=t(x)$. Third, $x'$ is fed to the classifier to get the softmax output denoted by $\bbg(x')$. Fourth, $x$ is classified as adversarial if the softmax outputs $\bbg(x)$ and $\bbg(x')$ are significantly different. Formally, similarity between $\bbg(x)$ and $\bbg(x')$ is measured using the K-L divergence measure, denoted by $D_{\text{KL}}(\bbg(x),\bbg(x')$ and defined as follows: $D_{\text{KL}}(\bbg(x),\bbg(x') = \sum_{c\in\ccalC}\bbg_c(x)\log\frac{\bbg_c(x)}{\bbg_c(x')}$
where $\bbg_c(x)$ denotes the $c$-th entry in the softmax output vector $\bbg(x)$; in other words, $\bbg_c(x)$ can be viewed as the probability that the class of image $x$ is $c$.  Specifically, if
\begin{equation}\label{eq:Jxx}
    J(x,x')=\min(D_{\text{KL}}(\bbg(x),\bbg(x')),D_{\text{KL}}(\bbg(x'),\bbg(x)))
\end{equation} 
is greater than a threshold $\tau$, then $x$ is classified as an adversarial input, i.e., $\hat{f}_d(x)=1$ (i.e., $x$ is \text{adversarial}); otherwise, $x$ is classified as a legitimate/clean image, i.e., $\hat{f}_d(x)=\text{clean}$. In \cite{nesti2021detecting}, it is mentioned that performance of VG can be improved if the max, instead of the min, operator is employed in \eqref{eq:Jxx}. Hereafter, to stay consistent with \cite{kantaros2021real}, we use the min operator. The detection threshold is selected using Receiver Operating Characteristics (ROC) graphs. Given an ROC graph computed on a validation set, a threshold that yields high true positive and low false alarm rate is selected.  
\subsection{$\text{VisionGuard}^*$: Detecting Attacks in Time-Series Image Data}\label{sec:timeSeriesDet}


In what follows, we build a new modular time-series detector,  called $\text{VisionGuard}^*$ ($\text{VG}^*$), that classifies an image $x_t$ as clean or adversarial, given a sequence $\ccalX_{0:t-1}$ of past images/observations. As mentioned in Section \ref{sec:detectPr}, hereafter, we assume that all images in the sequence $\ccalX_{0:t}$ contain the same object of interest.
%
Particularly, let $X_{0:t-1}=x_0,x_1,\dots,x_{t-1}$ be an ordered sequence of $t-1$ past observations/images $x_n$, $n\in\{0,1,\dots,t-1\}$ and $x_t$ denote the current input image. Also, let $T\geq0$ denote the length of a sliding detection window capturing how many past images from $X_{0:t}$ can be used. To determine whether $x_t$ is an adversarial image, we examine what is the output of $\hat{f}_d$ on the last $T$ images (including $x_t$). Then a majority-vote mechanism is employed to reason about $x_t$; see Alg. \ref{alg:det2}. Formally, the proposed detector is constructed as follows. 
%
First, we compute the average output of $\hat{f}_d$ over the past $T$ images as follows:\footnote{Note that \eqref{eq:majVote} can be written more generally as the weighted average of the detector outputs where e.g., a larger weight is assigned to more important frames. For instance, in an autonomous driving setup, a larger weight to collected images as the car approaches the object of interest (e.g., a traffic sign).}
\begin{equation}\label{eq:majVote}
  s(x_t) =
    \begin{cases}
      \frac{\sum_{n=t-T}^t \hat{f}_d(x_n)}{T} & \text{if}~ t\geq T\\
      \frac{\sum_{n=0}^t \hat{f}_d(x_n)}{T} & \text{otherwise}
    \end{cases}       
\end{equation}



Observe that if $s(x_t)\geq 0.5$, then this means that the majority of the past  $T$ images, including the current image $x_t$, are classified as adversarial. Otherwise, the majority is classified as clean. Thus, given $s(x_t)$, we define a time-series detector $f_d$ that complies with majority vote, i.e.,
\begin{equation}\label{eq:VGstar}
    f_d(x_t|X_{t-T:t-1}) =
    \begin{cases}
     1 ~(\text{i.e., adversarial}) & \text{if}~ s(x_t)\geq 0.5\\
     0 ~(\text{i.e., clean}) & \text{otherwise}
    \end{cases}       
\end{equation}
\begin{rem}[Variants of the Time-Series Detector]
We note that $\text{VG}^*$ can be extended in various ways. First, any single image detector can be used in place of VG to reason about a single image; see Section \ref{sec:sim1}. Second, to reason about a single image frame $x_n$, majority-vote over the output of an ensemble of single-image detectors can be applied as opposed to simply relying to the VG output. Alternatively, an ensemble of DNN image classifiers can be used (as opposed to a single one considered in this paper) as in \cite{pang2019improving}. Once a single image $x_n$ has been classified as adversarial or clean using any of the above ways, majority vote can be applied as in \eqref{eq:majVote}-\eqref{eq:VGstar}. 
\end{rem}


\begin{algorithm}[t]
   \caption{$\text{VisionGuard}^*$: Time-Series Detector}
   \label{alg:det2}
\begin{algorithmic}
   \STATE {\bfseries Input:} \{Length $T$ of detection window; Sequence of images $X_{t-T:t}$; Single-Image detector $\hat{f}_d$\}
   \STATE {\bfseries Output:} \{$f_d(x_t|X_{t-T:t-1})=1$ if $x_t$ is adversarial, and $f_d(x_t|X_{t-T:t-1})=0$ otherwise\}
   \STATE Feed images $x_n$ to $\hat{f}_d$ for all $n\in\{t-T, t-T+1, \dots, t\}$  \label{det2:fd}
   \STATE Compute $s(x_t)$ as in  \eqref{eq:majVote} \label{det2:maj}
   \STATE Initialize $f_d(x_t|X_{t-T:t-1})=0$
   \IF{$s(x_t)\geq 0.5$}
   \STATE $f_d(x_t|X_{t-T:t-1})=1$
   \ENDIF
\end{algorithmic}
\end{algorithm}
\subsection{Enhancing Detection Performance via Historical Data}\label{sec:Prop}
\vspace{-0.1cm}
In this section, we provide conditions for the single-image detector $\hat{f}_d$ and the length $T$ of the sliding detection window that need to be satisfied so that the proposed time-series detector $f_d$ achieves better detection performance than $\hat{f}_d$. These conditions are formally stated in the following proposition; see also Fig. \ref{fig:boundT}. 

\begin{prop}\label{prop:betterPerf}
Let $x_t$ denote an image under investigation and $X_{0:t-1}$ be a time-ordered sequence of images. Assume that $x_t$ and $X_{0:t-1}$ contain the same object of interest. Also let $\hat{f}_d:\ccalX\rightarrow\{0,1\}$ denote a single-image detector with probability $\hat{p}$ of correct classification/output. Then if (i) the outputs of $\hat{f}_d$ over all images in $X_{0:t-1}$ and $x_t$ are independent; (ii) $\hat{p}>0.5$; and (iii) the number $L$ of available past images in $X_{0:t-1}$ to reason about $x_t$
satisfies the following condition:\footnote{Notice that by construction of $f_d$, we have that $L$ is defined as follows: $L=\min\{t, T\}$, i.e., $L=T$ if $t\geq T$ and $L=t$ otherwise.}
\begin{align}\label{eq:boundT}
    L> \frac{2[\ln(2)-\ln(1-\hat{p})]}{(2\hat{p}-1)^2},
\end{align}
then, it holds that $p>\hat{p}$, where $p$ denotes the probability of correct classification of the time-series detector $f_d$. 
\end{prop}

\begin{proof}
This result is shown by applying Hoeffding's inequality.
Let $X_{t-T:t}$ be the ordered sub-sequence of images that lie within the sliding detection window of $f_d$. 
We define random variables $Y_n$ that are $1$, if $\hat{f}_d$ classifies the corresponding image $x_n$ correctly, and $0$ otherwise, for all $n\in\{t-T,\dots,t\}$. 
\begin{figure}[t]
  \centering
    \includegraphics[width=0.6  \linewidth]{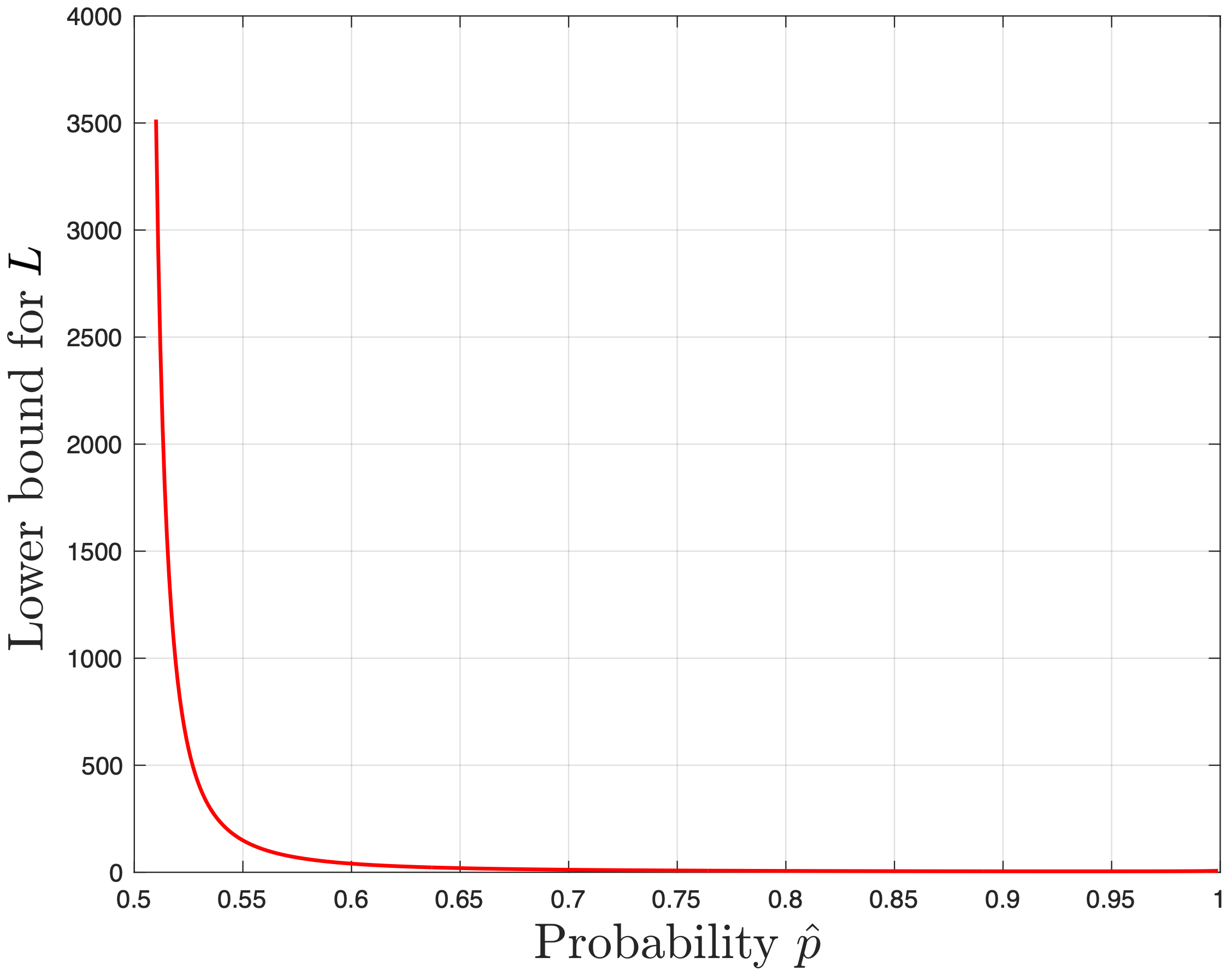}
    \vspace{-0.3cm}
  \caption{Graphical illustration of Proposition \ref{prop:betterPerf}. 
  The curve corresponds to the lower bound \eqref{eq:boundT}. Observe that as $\hat{p}\to 1$, the lower bound for the length $L$ that needs to be satisfied so that the time series detector $f_d$ performs better than the single-image detector $\hat{f}_d$ decreases. For example, if $\hat{p}=0.6$ then $L> 40.24$, while if $\hat{p}=0.9$ then $L> 4.69$. }\vspace{-0.5cm}
  \label{fig:boundT}
  \vspace{-0.2cm}
\end{figure}
%
Next, we define $Y_{x_t}=\sum_{n=t-L}^t Y_n$, i.e., $Y_{x_t}$ is a random variable counting the number of correct classifications of $\hat{f}_d$ over the sub-sequence of images $X_{t-L:t}=x_{t-L}, x_{t-L+1},\dots,x_t$. By definition, $Y_{x_t}$ follows a binomial distribution with mean $T\hat{p}$. By construction of $f_d$ the probability that $f_d$ misclassifies $x_t$, satisfies: $\mathbb{P}(f_d(x_t)\neq c(x_t))=1-p=\mathbb{P}(Y_{x_t}<\frac{L}{2})$,
where $c(x_t)$ denotes the correct `adversarial' label of $x_t$, i.e., $c(x_t)\in\{\text{clean}, \text{adversarial}\}$. Thus, we have that:
$\mathbb{P}\left(f_d(x_t)\neq c(x_t)\right)=\mathbb{P}\left(Y_{x_t}<\frac{L}{2}\right)\leq \mathbb{P}\left(Y_{x_t}-L\hat{p}<\frac{L}{2}-L\hat{p}\right)\leq \mathbb{P}\left(|Y_{x_t}-L\hat{p}|>\frac{L(2\hat{p}-1)}{2}\right)= \mathbb{P}\left(|Y_{x_t}-\mathbb{E}(Y_{x_t})|>\frac{L(2\hat{p}-1)}{2}\right)\leq 2e^{-\frac{L(2\hat{p}-1)^2}{2}}$
where the last inequality is due to Hoeffding's inequality. Note that Hoeffding's inequality can be applied if (a) $Y_n$ are independent random variables, which holds by assumption (i) and (b) $\frac{L(2\hat{p}-1)}{2}>0$, which holds since, by assumption (ii), we have $\hat{p}>0.5$. Observe that the derived upper bound (i.e., $2e^{-\frac{L(2\hat{p}-1)^2}{2}}$) goes to $0$ as  $L$ goes to infinity. Next, using this upper bound, we show that if $L$ satisfies \eqref{eq:boundT} then the detection performance of $f_d$ is greater than the performance of $\hat{f}_d$, i.e., $p>\hat{p}$ or, equivalently, $1-p<1-\hat{p}$. In math, we show that the following inequality 
\begin{equation}\label{eq:show1}
    \mathbb{P}(f_d(x_t)\neq c(x_t))-\mathbb{P}(\hat{f}_d(x_t)\neq c(x_t)<0,
\end{equation}
where $\mathbb{P}(f_d(x_t)\neq c(x_t))=1-\hat{p}$, holds if \eqref{eq:boundT} is true. 
An over-approximate range of values for $L$ so that \eqref{eq:show1} is satisfied can be found by using the above derived upper bound. Specifically, in what follows, we compute values for $L$ so that $2e^{-\frac{T(2\hat{p}-1)^2}{2}}-(1-\hat{p})<0$.
Solving this inequality for $L$ yields \eqref{eq:boundT}.\footnote{Note that this requires to assume that $1-\hat{p}>0$, i.e., $\hat{p}<1$ so that $\ln(1-\hat{p})$ is well defined. Nevertheless, if $\hat{p}=1$, then it trivially holds that any length $L\geq0$ can be used to maintain the perfect performance of the single-image detector.}
Note that this lower bound is positive since $\ln(2)\approx0.69>0$ while $-\ln(1-p)>0$ due to $1-\hat{p}<1$. Consequently, if the outputs of the single image detector $\hat{f}_d$ are independent, $\hat{p}>0.5$, and $L$ satisfies \eqref{eq:boundT}, then $p>\hat{p}$ completing the proof.
%
%
\end{proof}
\begin{rem}
We note that if the single-image detector satisfies the conditions (i)-(iii) of Prop. \ref{prop:betterPerf}, then as $L\to\infty$, the probability of misclassification for the time-series detector goes to $0$. 
Also, note that Prop. \ref{prop:betterPerf}, implies that the length $T$ of the sliding window should be designed so that it is larger than the lower bound captured in \eqref{eq:boundT}. 
\end{rem}

\section{Experiments}\label{sec:sim1}

In this section, we extensively demonstrate the performance of $\text{VG}^*$ and compare its performance against baseline detectors \cite{lee2018simple,liang2018detecting} summarized in Section \ref{sec:sumDetExp}. Particularly, in Sections
\ref{sec:setupSingle}, we first discuss how these baseline time-series detectors are set up. In Section \ref{sec:evalDet}, we present the data collection process that involves recording videos of clean and adversarial traffic signs.
%
Then, in Sections \ref{sec:evalDet2}- \ref{sec:evalDet3}, we compare the performance of $\text{VG}^*$ on the above-mentioned videos against the baseline time-series detectors. 
Our experiments validate that historical data can enhance performance of single-image detectors, as shown in Proposition \ref{prop:betterPerf}. In summary, our comparative experiments show that the proposed detector is characterized by a more consistent performance than the competitive detectors whose performance changes significantly across attacks, datasets, and DNN models. Finally, we discuss the real-time performance of the proposed detectors. All experiments have been executed on a computer with Intel(R) Xeon(R) Gold 6148 CPU, 2.40GHz. 

\subsection{Summary of Considered Baseline Detectors}\label{sec:sumDetExp}
\textbf{Mahalanobis-based Detector (MD):}
This detector relies on computing a confidence score for each test sample; if this confidence score is above a threshold then the test image is considered abnormal or ouut-of-distribution \cite{lee2018simple}. To define this confidence score, the Mahalanobis distance to the training data set is computed. Specifically, $|\ccalC|$ class-conditional distributions are defined for each layer of the targeted DNN that are assumed to be Gaussian distributions. In math, for each layer $\ell$ of the DNN $f$, the empirical class mean $\mu_{c,\ell}$ and covariance $\Sigma_{\ell}$ of training samples $\{(x_i,y_i)\}_{i=1}^N$, where $y_i$ is the true label of $x_i$, are defined as follows: (i) $\mu_{c,\ell}=\frac{1}{N_c}\sum_{i:y_i=c}f_{\ell}(x_i)$; and $ \Sigma_{\ell}=\frac{1}{N_c}\sum_{c\in\ccalC}\sum_{i:y_i=c}(f_{\ell}(x_i)-\mu_{c,\ell})(f_{\ell}(x_i)-\mu_{c,\ell})^T$,
where $N_c$ is the total number of training samples with label $c$ and  $f_{\ell}(x_i)$ refers to the output of the $\ell$-th layer of the DNN $f$. Then, given a new sample $x$, for each layer $\ell$, the following steps are executed. First, the closest class $\hat{c}$ is computed as per the Mahalanobis distance: $\hat{c}=\argmin_c(f_{\ell}(x)-\hat{\mu}_{c,\ell})^T\hat{\Sigma}_{\ell}^{-1}(f_{\ell}(x_i)-\hat{\mu}_{c,\ell})$. Second, a small noise is added to $x$, i.e., $\hat{x}=x-\epsilon\nabla_x(f_{\ell}(x)-\hat{\mu}_{\hat{c},\ell})^T\hat{\Sigma}_{\ell}^{-1}(f_{\ell}(x_i)-\hat{\mu}_{\hat{c},\ell}))$. Third, a confidence score $M_{\ell}$ is computed as: $M_{\ell}=\max_c-(f_{\ell}(\hat{x})-\hat{\mu}_{c,\ell})^T\hat{\Sigma}_{\ell}^{-1}(f_{\ell}(\hat{x})-\hat{\mu}_{c,\ell})$. Once these steps are repeated for all layers $\ell$, the following confidence score is computed which is used for  detection: $M=\sum_\ell\alpha_\ell M_\ell$. This detector has been tested against digital attacks and OOD data \cite{lee2018simple}.

\textbf{Entropy-based Detector (ED):}
This detector relies on the observation that digital attacks embed an imperceptible amount of noise to images \cite{liang2018detecting}. 
Motivated by this observation, the adversarial perturbation is treated as artificial noise and leverages adaptive noise reduction techniques
to reduce its adversarial effect, where aggressiveness of the denoising strategy is selected based on the entropy of the image. 
%
%
%
In particular, this detector comprises the following steps. 
First, the entropy of an image $x$ with dimensions $M\times N$ is computed as: $H=-\sum_{i=0}^{255} p_i\log_2 p_i,$ 
where $p_i=f_i/(M\times N)$ and $f_i$ is the frequency of pixel level $i\in[0,1,\dots,255]$. For RGB images, the entropy is computed as the average of entropies over the three channels.
Next, the denoising process follows. First, the image gets quantized using a scalar quantization method where  all inputs (i.e., pixel values) within a specified interval are mapped to a common value (called codeword), and the inputs in a different interval will be mapped to a different codeword. The number of intervals for uniform quantization is determined using the entropy of the image. Next, the entropy of the quantized sample is computed to determine if spatial smoothing will be applied. Only if the entropy is larger than a threshold, spatial smoothing (e.g., using a box filter) is applied. The image generated by this adaptive denoising process is denoted by $T(x)$.
Finally, the original image $x$ and the filtered one $T(x)$ are sent to the target classifier. If the predicted labels for $x$ and $T(x)$ are the same, then $x$ is considered benign/clean. Otherwise, it is identified as adversarial.

\textbf{Time-Series Detectors $\text{MD}^*$ \& $\text{ED}^*$:} Following the same logic as in Section \ref{sec:timeSeriesDet}, we construct the time series detectors $\text{MD}^*$, and $\text{ED}^*$ based on the single-image detectors $\text{MD}$, and $\text{ED}$. Specifically, these detectors are built exactly as $\text{VG}^*$; the only difference is that VG has been replaced by MD and ED.

\subsection{Setting Up Time-Series Detectors}\label{sec:setupSingle}
In what follows, we present the target DNN models  and we discuss how we set up the single image detectors, discussed above, for the $\text{RP}_2$ attack on traffic signs. We have selected AdvNet as a training set to configure all detectors; AdvNet is a recently proposed dataset with clean and physically attacked traffic signs \cite{kantaros2021real}. All detectors are evaluated on a test set constructed by extracting frames of videos with clean and physically attacked traffic signs. More details follow.

\textbf{Target DNN Models:} 
Hereafter, in our evaluation, we consider the LISA-CNN \cite{eykholt2018robust} and the GTSRB-CNN  \cite{stallkamp2012man} that can classify traffic signs. Both CNNs take as input $32\times 32$ RGB images while LISA and GTSRB CNN have $17$ and $43$ labels/outputs, respectively. LISA-CNN has
91\% accuracy on the LISA test set and the GTSRB-CNN achieves 95.7\% accuracy on the GTSRB test set.

\textbf{VG:} As discussed in Section \ref{sec:ImageDetector}, to select a detection threshold $\tau$ for VG, we construct the ROC curve of VG on AdvNet, which captures the trade-off between the true positive (TP) and false positive/alarm (FP) rate for a range of thresholds. Then, based on the ROC curve, we select a threshold that yields a high TP rate and low FP rate. For instance, the ROC curve when the target neural network is the LISA CNN is shown in Figure \ref{fig:ROC1}; the area under this ROC (AUROC) is $89.43\%$. Then, we pick a threshold $\tau=0.43$ which yields a TP rate equal to $80\%$ and a FP rate equal to $19\%$. Following the same process, we select $\tau=0.78$ for the GTSRB CNN. Note that for both CNNs, VG applies brightness transformation with $\text{HSV}=200$.

\begin{figure}[t]
  \centering
    \includegraphics[width=0.55\linewidth]{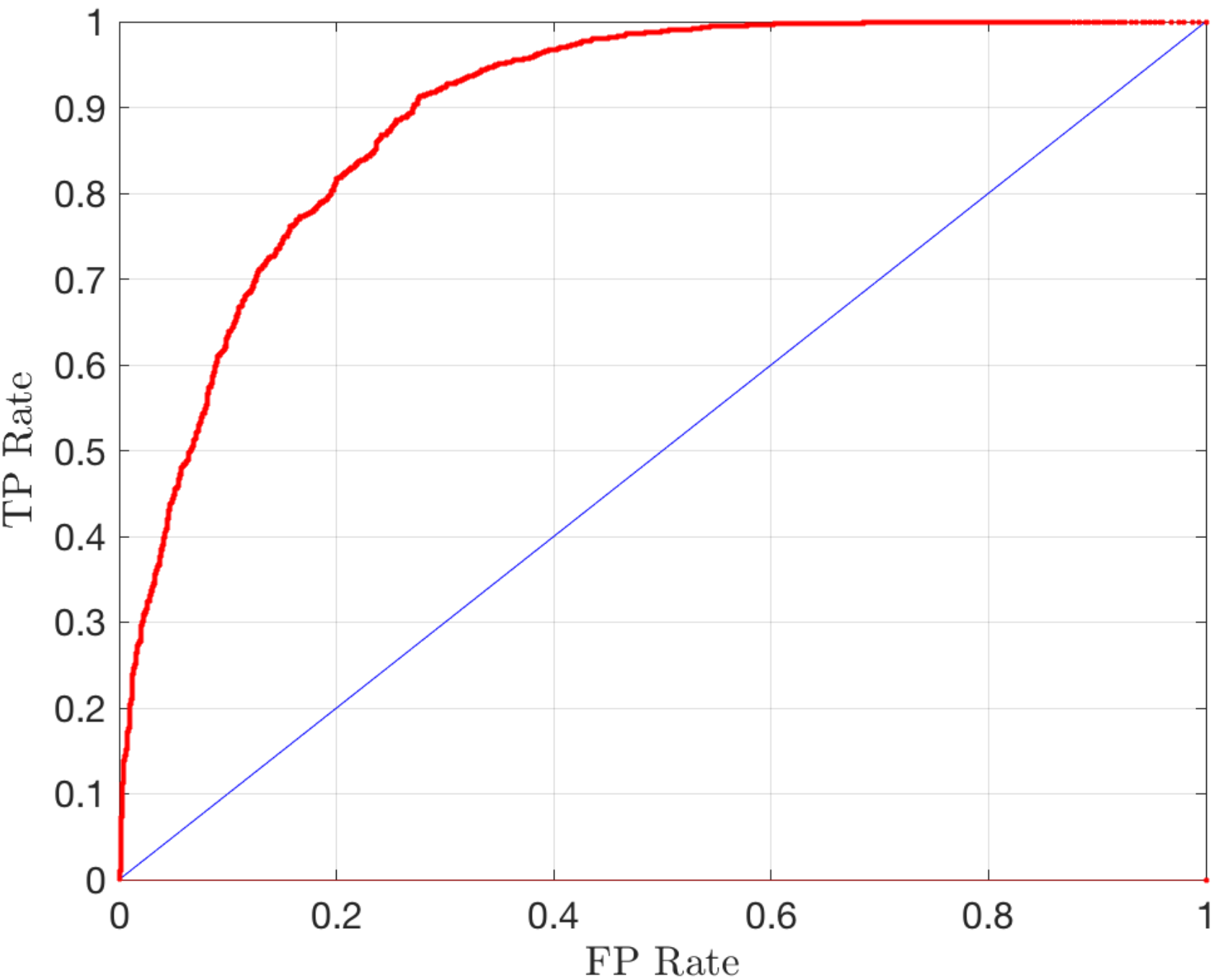}
    \vspace{-0.4cm}
  \caption{ROC curve (red) corresponding to VG equipped with brightness transformation with $\text{HSV}=200$ on AdvNet for LISA CNN \cite{eykholt2018robust}. The detection threshold $\tau=0.43$ yields $\text{TP rate}=80\%$ and $\text{FP rate}=19\%$ on AdvNet. Each point in the ROC curve corresponds to a detection threshold while the blue line represents the ROC of a random detector.}\vspace{-0.7cm}
\label{fig:ROC1}
\end{figure}

\textbf{MD:} To train the Mahalanobis detector, we followed the same process as we did for VG. First, we compute the ROC of this detector using  AdvNet. Then, using the ROC curve, we pick a threshold that yields a high TP and a low FP rate. The AUROC of MD for the LISA CNN is $81.2\%$ which is smaller than the AUROC of VG. In fact, given any fixed TP rate, the corresponding FP rate of MD is higher than that of VG. For instance, for $\text{TP rate}=80\%$, MD yields a FP rate that is approximately $45\%$ (while for VG, $\text{TP rate}=19\%$). Thus,  prioritizing a low FP rate, we select a threshold which yields $\text{TP rate}=50\%$ and $\text{FP rate}=8\%$. Nevertheless, the choice of the threshold is user- and scenario- specific.


\textbf{ED:} This detector requires the selection of a spatial smoothing filter \cite{liang2018detecting}. In \cite{liang2018detecting}, it was shown that $5\times 5$ cross, $7\times 7$ cross, $5\times 5$ diamond, $7\times 7$ diamond, and $5\times 5$ box filters yield satisfactory performance against digital attacks. To pick a filter for physical attacks, we evaluated this detector on AdvNet for various cross, box and diamond filters. The best performance was attained using a $3\times 3$ box filter.

\subsection{Data Collection \& Evaluation Metric}\label{sec:evalDet}
To evaluate the detection performance of $\text{VG}^*$ against the baseline detectors $\text{MD}^*$ and $\text{ED}^*$, we have recorded videos of clean and physically attacked stop, yield, and speed limit 35 signs. The traffic signs have been attacked in the same way as in AdvNet. The videos have an average duration of $12$ seconds and they were taken during daytime by a 12MP smartphone camera at $30$ frames-per-second. Each video contains a single (clean or attacked) traffic sign. At the beginning of the videos, the distance between the camera and the traffic sign is on average $15$m which gradually decreases. Snapshots from a representative video are provided in Fig. \ref{fig:snapshotsVideo}. 
To compare the detectors, we extract the frames from each video (to construct the time-series data $\ccalX_{0:t}$) and we report the accuracy of each detector defined as follows:
\begin{equation}\label{eq:acc1}
    \text{Acc}(T)=\frac{\text{number of correctly labeled images $x_t$}}{\text{total number of images}},
\end{equation}
where the images $x_t$ are labeled as `clean' or `adversarial' using the past $T$ frames in the video.
\vspace{-0.1cm}
\subsection{Comparative Evaluation on LISA CNN}\label{sec:evalDet2}
\vspace{-0.05cm}
We evaluate the performance of the time-series detectors  $\text{VG}^*$, $\text{MD}^*$, and $\text{ED}^*$ on the recorded videos mentioned in Section \ref{sec:evalDet}. 
We use these time-series detectors to classify each frame $x_t$ as clean or adversarial. Note that there can be frames without the traffic sign. Thus,  the detector uses the past $T$ frames that contained the sign. To reason about whether a traffic sign exists within a frame or not, we train a YOLO v3 object detector - using clean images from AdvNet - that can detect stop, yield and speed limit 35 signs. All detectors have been evaluated on the same videos. The results are summarized in Tables \ref{tab:VGcleanLISA}-\ref{tab:EntrAdvLISA} for LISA CNN. 
Particularly, we report the accuracy of classifying correctly, as adversarial or clean, the frames on each video for $T\in\{0,3,10,T_{\text{all}}\}$ as per \eqref{eq:acc1}.  Notice that when $T=0$, the time-series detector coincides with the corresponding single-image detector. The accuracy reported on $T=0$ can also be viewed as an approximation of the accuracy, denoted by $\hat{p}$ in Proposition \ref{prop:betterPerf}, of the single-image detector computed using the frames of the corresponding video.
\textit{Also, by $T=T_{\text{all}}$, we refer to the case where $T$ is large enough so that the time-series detector considers all past frames in $\ccalX_{0:t-1}$ of the video to reason about $x_t$, for all $t\geq 0$.} 

\begin{figure}[t]
  \centering
    \includegraphics[width=0.6\linewidth]{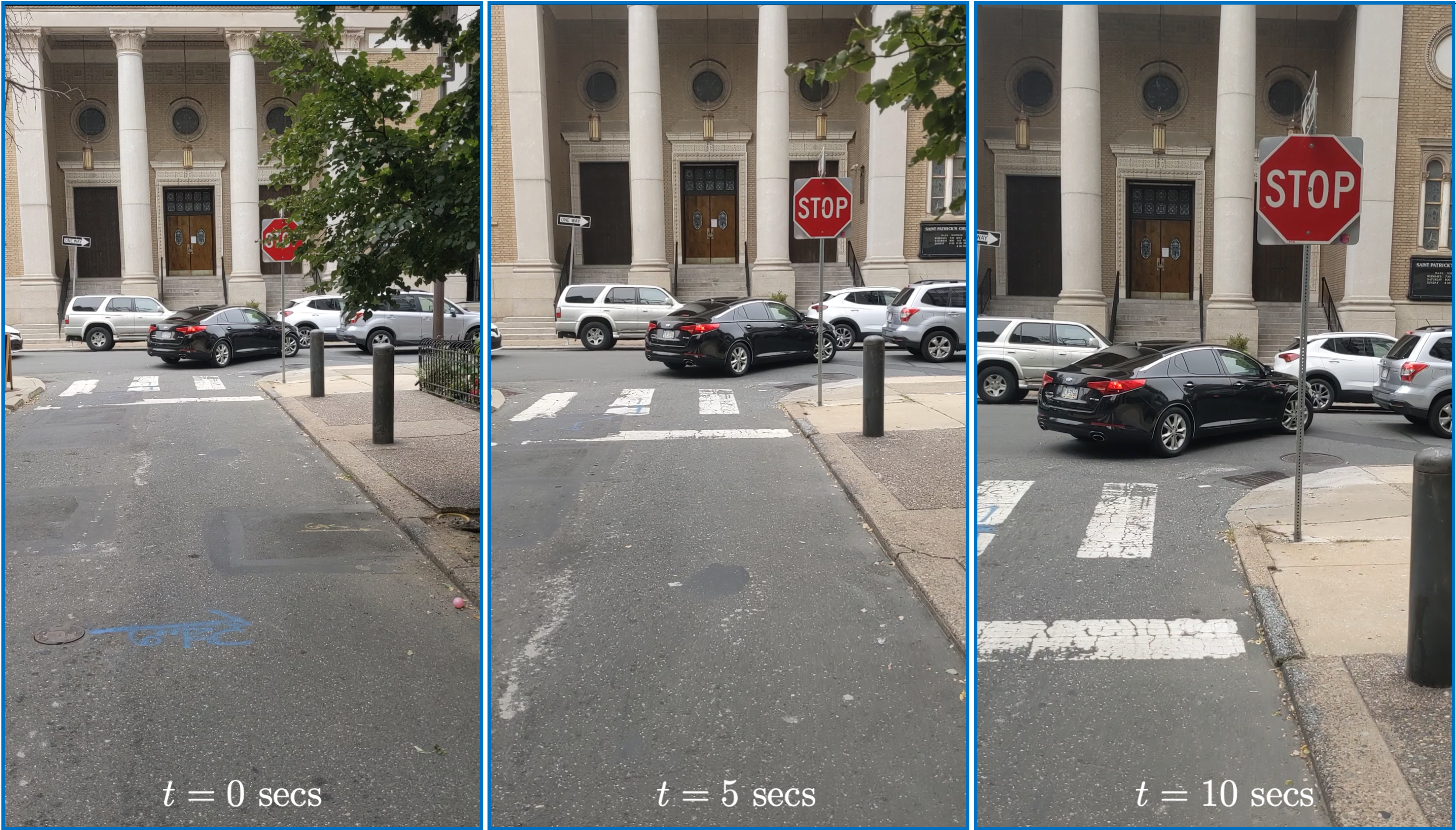}\vspace{-0.2cm}
  \caption{Snapshots of a recorded video containing a clean stop sign.}
\label{fig:snapshotsVideo}\vspace{-0.5cm}
\end{figure}

\textbf{$\text{VG}^*$}: In Tables \ref{tab:VGcleanLISA}-\ref{tab:VGadvLISA}, we show the average detection performance of $\text{VG}^*$ across five videos of clean traffic signs considering the LISA CNN.  
We observed that, in general, if the accuracy of the single-image detector $\hat{f}_d$ on each frame is reasonably high (implying that $\hat{p}>0.5$ as estimated by the $T=0$ columns), then when $T$ is very large, the time-series detector $f_d$ achieves high performance on distinguishing clean (Table \ref{tab:VGcleanLISA}) and adversarial (Table \ref{tab:VGadvLISA}) traffic signs. This is also implied by Proposition \ref{prop:betterPerf}. 

\begin{table}[t]
\caption{$\text{VG}^*$: LISA CNN And \textit{CLEAN} Traffic Signs}\vspace{-0.3cm}
\label{tab:VGcleanLISA}
\begin{tabular}{|l|l|l|l|l|l|l|}
\hline
        & T=0    & T=3  & T=10 & T=30 & $T_{\text{all}}$ & \begin{tabular}[c]{@{}l@{}}DNN\\ Acc\end{tabular} \\ \hline
Stop    & 87.4\% & 89\%   & 88.6\% & 90.4\% & 93\%   & 93.8\%                                            \\ \hline
Speed35 & 85.4\%   & 87.2\% & 87.4\% & 95.8\% & 98.6\% & 93.8\%                                            \\ \hline
Yield   & 92.6\%   & 95\%   & 97\%   & 100\%  & 100\%  & 99.2\%                                              \\ \hline
\end{tabular}
\vspace{-0.5cm}
\end{table}

\begin{table}[t]
\caption{$\text{VG}^*$: LISA CNN And \textit{ADVERSARIAL} Traffic Signs}\vspace{-0.3cm}
\label{tab:VGadvLISA}
\begin{tabular}{|l|l|l|l|l|l|l|}
\hline
        & T=0    & T=3  & T=10 & T=30 & $T_{\text{all}}$ & \begin{tabular}[c]{@{}l@{}}DNN\\ Acc\end{tabular} \\ \hline
Stop    & 97.3\%        & 98.8\%  & 99.8\% & 99.8\% & 99.8\% & 1.6\%                                            \\ \hline
Speed35  & 95.8\%        & 96.6\%  & 97.2\% & 97.8\% & 97.8\% & 0.5\%                                                                            \\ \hline
Yield   & 86.3\%        & 88.5\%  & 89.8\% & 90.2\% & 83.3\% & 22.3\%                                                                             \\ \hline
\end{tabular}
\vspace{-0.5cm}
\end{table}

\textbf{$\text{MD}^*$}: The average performance of $\text{MD}^*$ across five videos for LISA CNN is summarized in Tables \ref{tab:MahCleanLISA}-\ref{tab:MahAdvLISA}. Observe that this detector performs quite well on distinguishing clean and adversarial stop and speed limit signs but it performs poorly on yield signs. A potential reason for this is that MD assumes that class-conditions distributions are Gaussian distributions across all layers of the target neural network. This assumption was validated for digitally attacked and OOD data in \cite{lee2018simple}. Nevertheless, it may not hold for certain physical attacks, as the one considered here. 
Notice that as $T$ increases, the performance of $\text{MD}^*$ does not increase for the yield sign. However, this does not contradict Proposition \ref{prop:betterPerf} as the reported accuracy may be inaccurate (as only a few tens of image frames are used). Recall also that Proposition \ref{prop:betterPerf} does not imply that the detection performance should monotonically increase as $T$ increases. Overall, using the whole image sequence/video, the time-series detector classifies the yield sign (on average across all videos) as clean.
%
In total, the detection performance of $\text{MD}^*$ and $\text{VG}^*$ is comparable on stop signs but $\text{VG}^*$ outperforms $\text{MD}^*$ on speed limit/yield signs.  

\textbf{$\text{ED}^*$}: The average performance of $\text{ED}^*$ across five videos for LISA CNN  is summarized in Tables \ref{tab:EntrCleanLISA}-\ref{tab:EntrAdvLISA}. Similar to $\text{MD}^*$, observe that the single-image detector ($T=0$) performs much worse than VG in several videos. The reason is because ED is built using filters aiming to remove digital noise as opposed to e.g., VG that uses brightness transformations to deal with adversarial stickers. As a result, the performance of this detector is satisfactory on clean traffic signs (see Table \ref{tab:EntrCleanLISA}) but not on physically attacked traffic signs (see Table \ref{tab:EntrAdvLISA}). Overall, this detector seems to performs quite well on distinguishing clean and adversarial stop signs but it fails to generalize to speed limit and yield signs. 
%

\begin{table}[t]
\caption{$\text{MD}^*$: LISA CNN And \textit{CLEAN} Traffic Signs}\vspace{-0.3cm}
\label{tab:MahCleanLISA}
\begin{tabular}{|l|l|l|l|l|l|l|}
\hline
        & T=0    & T=3  & T=10 & T=30 & $T_{\text{all}}$ & \begin{tabular}[c]{@{}l@{}}DNN\\ Acc\end{tabular} \\ \hline
Stop    & 97.4\%     & 99.4\%       & 99.6\%     & 99.8\%     & 99.8\%       & 93.8\%                                                             \\ \hline
Speed35 & 72\%     & 80.4\%     & 84.2\%     & 93.8\%     & 99.2\%     & 93.8\%                                                                    \\ \hline
Yield    & 61.8\%     & 82.4\%       & 51.8\%       & 54.6\%      & 60\%      & 99.2                                                                \\ \hline
\end{tabular}
\vspace{-0.5cm}
\end{table}

\begin{table}[t]
\caption{$\text{MD}^*$: LISA CNN And \textit{ADVERSARIAL} Traffic Signs}\vspace{-0.3cm}
\label{tab:MahAdvLISA}
\begin{tabular}{|l|l|l|l|l|l|l|}
\hline
        & T=0    & T=3  & T=10 & T=30 & $T_{\text{all}}$ & \begin{tabular}[c]{@{}l@{}}DNN\\ Acc\end{tabular} \\ \hline
Stop    & 83.7\%        & 89\%  & 95.2\% & 95.2\% & 98.3\% & 1.6\%                                                                                              \\ \hline
Speed35 & 82.2\%        & 89.2\%  & 95.8\% & 98.4\% & 99.2\% & 0.5\%                                                                                                        \\ \hline
Yield    & 34\%        & 33\%  & 33.4\% & 33.3\% & 33.3\% & 22.3\%                                                                                                     \\ \hline
\end{tabular}
\vspace{-0.5cm}
\end{table}

\begin{table}[t]
\caption{$\text{ED}^*$: LISA CNN And \textit{CLEAN} Traffic Signs}\vspace{-0.3cm}
\label{tab:EntrCleanLISA}
\begin{tabular}{|l|l|l|l|l|l|l|}
\hline
        & T=0    & T=3  & T=10 & T=30 & $T_{\text{all}}$ & \begin{tabular}[c]{@{}l@{}}DNN\\ Acc\end{tabular} \\ \hline
Stop    & 86.8\%        & 90.8\%  & 91.2\% & 93.8\% & 93.8\% & 93.8\%                       \\ \hline
Speed35 & 48.2\%        & 47\%    & 44.2\% & 44.6\% & 40.4\% & 93.8\%                       \\ \hline
Yield   & 95.8\%        & 98.6\%  & 99.2\% & 99.6\% & 99.6\% & 99.2\%                        \\ \hline
\end{tabular}
\vspace{-0.5cm}
\end{table}

\begin{table}[t]
\caption{$\text{ED}^*$: LISA CNN And \textit{ADVERSARIAL} Traffic Signs}
\vspace{-0.3cm}
\label{tab:EntrAdvLISA}
\begin{tabular}{|l|l|l|l|l|l|l|}
\hline
        & T=0    & T=3  & T=10 & T=30 & $T_{\text{all}}$ & \begin{tabular}[c]{@{}l@{}}DNN\\ Acc\end{tabular} \\ \hline
Stop    & 92.7\%        & 94.7\%  & 97.5\% & 99.7\% & 100\%  & 1.6\%                                                    \\ \hline
Speed35 & 34\%          & 31\%    & 33\%   & 26.6\% & 11.2\% & 0.5\%                                                          \\ \hline
Yield   & 68.5\%        & 70\%    & 73.8   & 72.7   & 77.2   & 22.3\%                                                          \\ \hline
\end{tabular}
\vspace{-0.5cm}
\end{table}

\textbf{Summary of comparative detection performance:} The comparative performance of the resulting time-series detector for the LISA CNN is summarized in Figure \ref{fig:LISAsum}. Specifically, Figure \ref{fig:cleanL} and \ref{fig:advL} illustrate the average performance of the three time-series detectors across all clean and adversarial traffic signs, respectively, for $T\in\{0,3,10,T_{\text{all}}\}$.  Observe that the VG-based time-series detector achieves the best performance on average, since it has been designed to be effective against physical attacks as opposed to \cite{lee2018simple,liang2018detecting} that focus on digital attacks and/or OOD data. Also, notice that among all considered single-image detectors (i.e., $T=0$), VG attains the best performance in terms of both the accuracy metric defined in \eqref{eq:acc1} and the AUROC; see Section \ref{sec:setupSingle}. 
%


\begin{figure}[t]
\centering
  \subfigure{
    \label{fig:imgNet1}
    \includegraphics[width=0.3\linewidth]{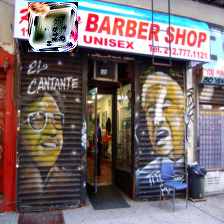}}
     \subfigure{
    \label{fig:imgNet2}
    \includegraphics[width=0.31\linewidth]{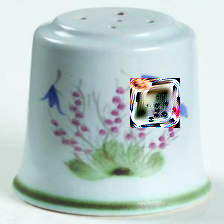}}\vspace{-0.4cm}
  \caption{Examples of ImageNet images attacked by \cite{brown2017adversarial} for the target label $y^*=\text{toaster}$.}\vspace{-0.4cm}
\label{fig:AdvPatchImageNet}
\end{figure}

\begin{figure}[t]
  \centering
  \subfigure[]{
    \label{fig:cleanL}
    \includegraphics[width=0.46\linewidth]{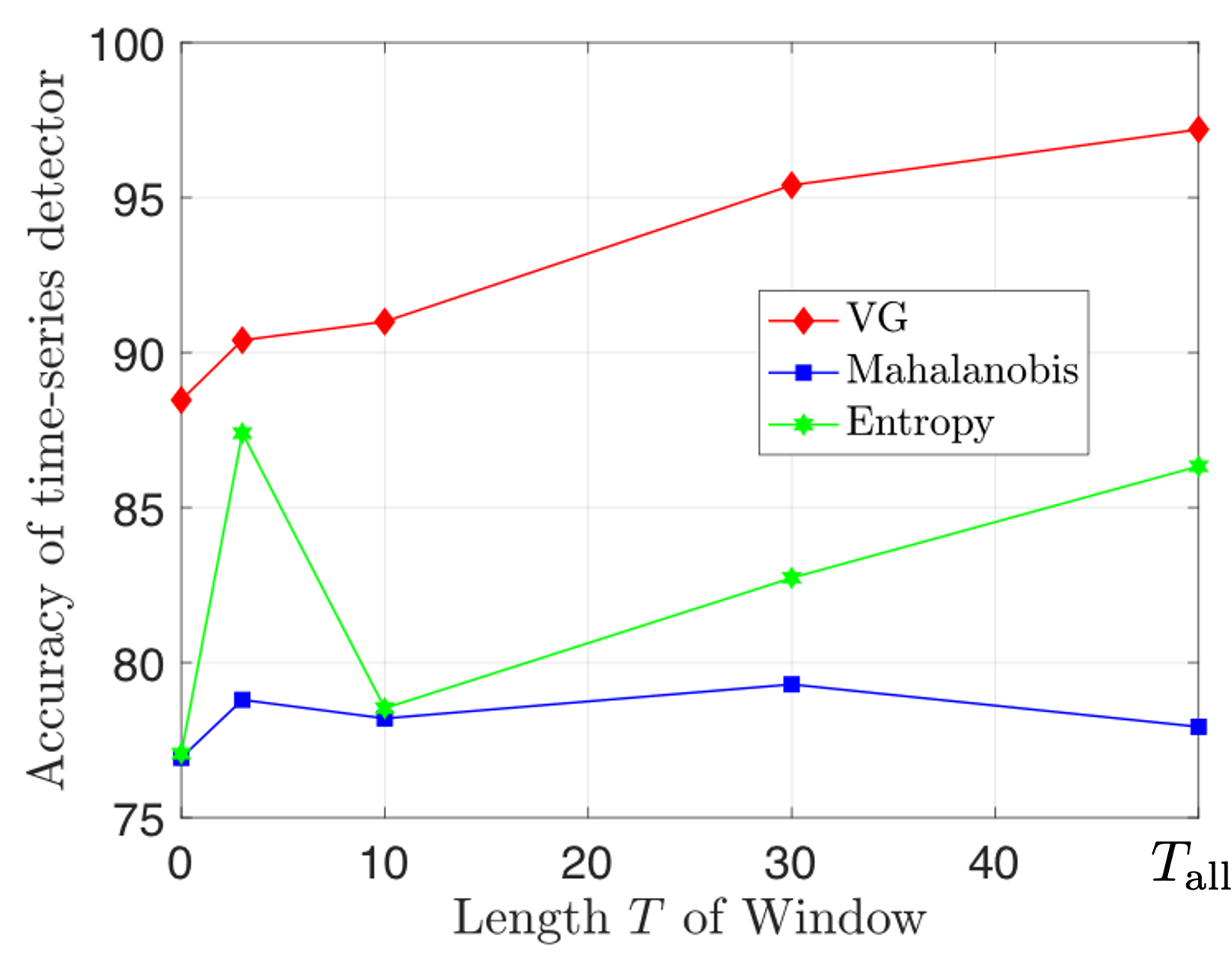}}
     \subfigure[]{
    \label{fig:advL}
    \includegraphics[width=0.46\linewidth]{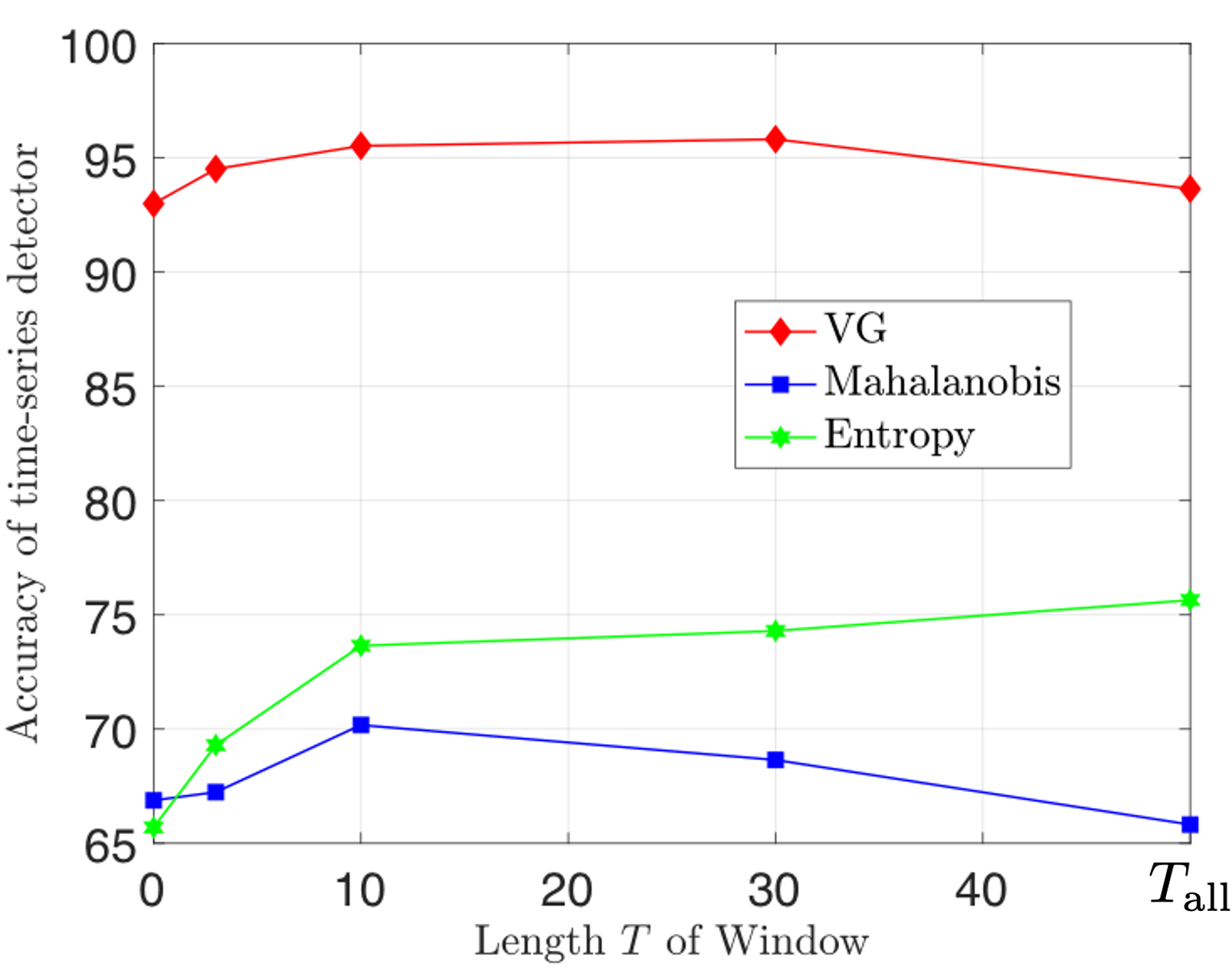}}\vspace{-0.3cm}
  \caption{
  Comparison of the constructed time-series detectors against clean (Fig. \ref{fig:cleanL}) and adversarial traffic signs (Fig. \ref{fig:advL}) for the LISA CNN.
  }\vspace{-0.2cm}
  \label{fig:LISAsum}
\end{figure}

\begin{table}[t]
\caption{$\text{VG}^*$: GTSRB CNN - (clean \& adversarial) stop signs}\vspace{-0.3cm}
\label{tab:VGgtsrb}
\begin{tabular}{|l|l|l|l|l|l|l|}
\hline
        & T=0    & T=3  & T=10 & T=30 & $T_{\text{all}}$ & \begin{tabular}[c]{@{}l@{}}DNN\\ Acc\end{tabular} \\ \hline
Clean Stop     & 85.8\%        & 89.8\%  & 91.4\% & 93\%   & 93\%   & 82.8\%                                           \\ \hline
Adv Stop & 77.2\%        & 81\%    & 87.7\% & 88.7\% & 97.2\% & 13.9\%                                         \\ \hline
\end{tabular}
\vspace{-0.5cm}
\end{table}

\begin{table}[t]
\caption{$\text{MD}^*$: GTSRB CNN - (clean \& adversarial) stop signs}\vspace{-0.3cm}
\label{tab:MahGTSRB}
\begin{tabular}{|l|l|l|l|l|l|l|}
\hline
        & T=0    & T=3  & T=10 & T=30 & $T_{\text{all}}$ & \begin{tabular}[c]{@{}l@{}}DNN\\ Acc\end{tabular} \\ \hline
Clean Stop      & 73.8\%        & 78\%    & 76.6\% & 82.2\% & 83.6\% & 82.8\%                                            \\ \hline
Adv. Stop  & 48.5\%          & 47.6\%  & 49\% & 50\% & 50\% & 13.9\%                                       \\ \hline
\end{tabular}
\vspace{-0.5cm}
\end{table}

\begin{table}[t]
\caption{$\text{ED}^*$: GTSRB CNN - (clean \& adversarial) stop signs}\vspace{-0.3cm}
\label{tab:EntrGTSRB}
\begin{tabular}{|l|l|l|l|l|l|l|}
\hline
        & T=0    & T=3  & T=10 & T=30 & $T_{\text{all}}$ & \begin{tabular}[c]{@{}l@{}}DNN\\ Acc\end{tabular} \\ \hline
Clean Stop     & 91.2\%        & 94\%    & 94.6\% & 97.4\% & 99.6\% & 82.8\%                                           \\ \hline
Adv. Stop  & 86\%          & 89.8\%  & 87.7\% & 93.8\% & 98.3\% & 13.9\%                                     \\ \hline
\end{tabular}
\vspace{-0.3cm}
\end{table}

\vspace{-0.1cm}
\subsection{Comparative Evaluation on GTSRB CNN}\label{sec:evalDet3}
\vspace{-0.1cm}
Next, we compare the performance of $\text{VG}^*$ against the time-series detectors $\text{MD}^*$ and $\text{ED}^*$, when the GTSRB-CNN is considered. 
The GTSRB-CNN is trained on the German Traffic Sign Recognition Benchmark. Since we did not have access to German traffic signs for our physical experiments, GTSRB-CNN is evaluated only on videos with US stop signs. The AUROC of VG (with brightness transformation and $\text{HSV}=200$) and MD for GTSRB-CNN on the AdvNet dataset is $97.81\%$ and and $79.9\%$, respectively. To set up the time series detectors, based on these ROCs, we select thresholds for VG and MD that yield $\text{TP rate}=95\%$, $\text{FP rate}=4\%$ and $\text{TP rate}=50\%$, $\text{FP rate}=8\%$. $\text{ED}^*$ is set up using a $3\times 3$ box filter (exactly as in Section \ref{sec:setupSingle}.  The comparative performance of the time-series detectors $\text{VG}^*$, $\text{MD}^*$, and $\text{ED}^*$ when the GTSRB-CNN is employed is summarized in Tables \ref{tab:VGgtsrb}-\ref{tab:EntrGTSRB}.
As in Section \ref{sec:evalDet2}, these tables report the accuracy, as defined in \ref{eq:acc1}, of the considered time-series detectors for videos containing a clean or adversarial stop sign.  The stickers were placed so that the GTSRB-CNN classifies stop signs as speed limit 35 signs as in \cite{kantaros2021real,eykholt2018robust}. Figure \ref{fig:GTSRBsum} visualizes the average performance of $\text{VG}^*$, $\text{MD}^*$, and $\text{ED}^*$  across all clean and adversarial videos reported in Tables \ref{tab:VGgtsrb}-\ref{tab:EntrGTSRB}. Observe first that $\text{ED}^*$ achieves the best (average) performance in distinguishing clean and adversarial stop signs when the GTSRB-CNN is considered. In fact, $\text{ED}^*$ achieved a satisfactory performance on stop signs when the LISA-CNN was considered but it failed to generalize to other traffic signs (see Section \ref{sec:evalDet2}). $\text{VG}^*$ performs comparably to $\text{ED}^*$ for the GTSRB-CNN but it outperforms it when the LISA CNN is considered. $\text{MD}^*$ achieved the worst performance in the GTSRB-CNN setup. Overall, we observe that $\text{VG}^*$  achieves a more consistent and satisfactory detection performance as the DNN model and the dataset change.
\begin{figure}[!h]
  \centering
  \subfigure[]{
    \label{fig:cleanG}
    \includegraphics[width=0.45\linewidth]{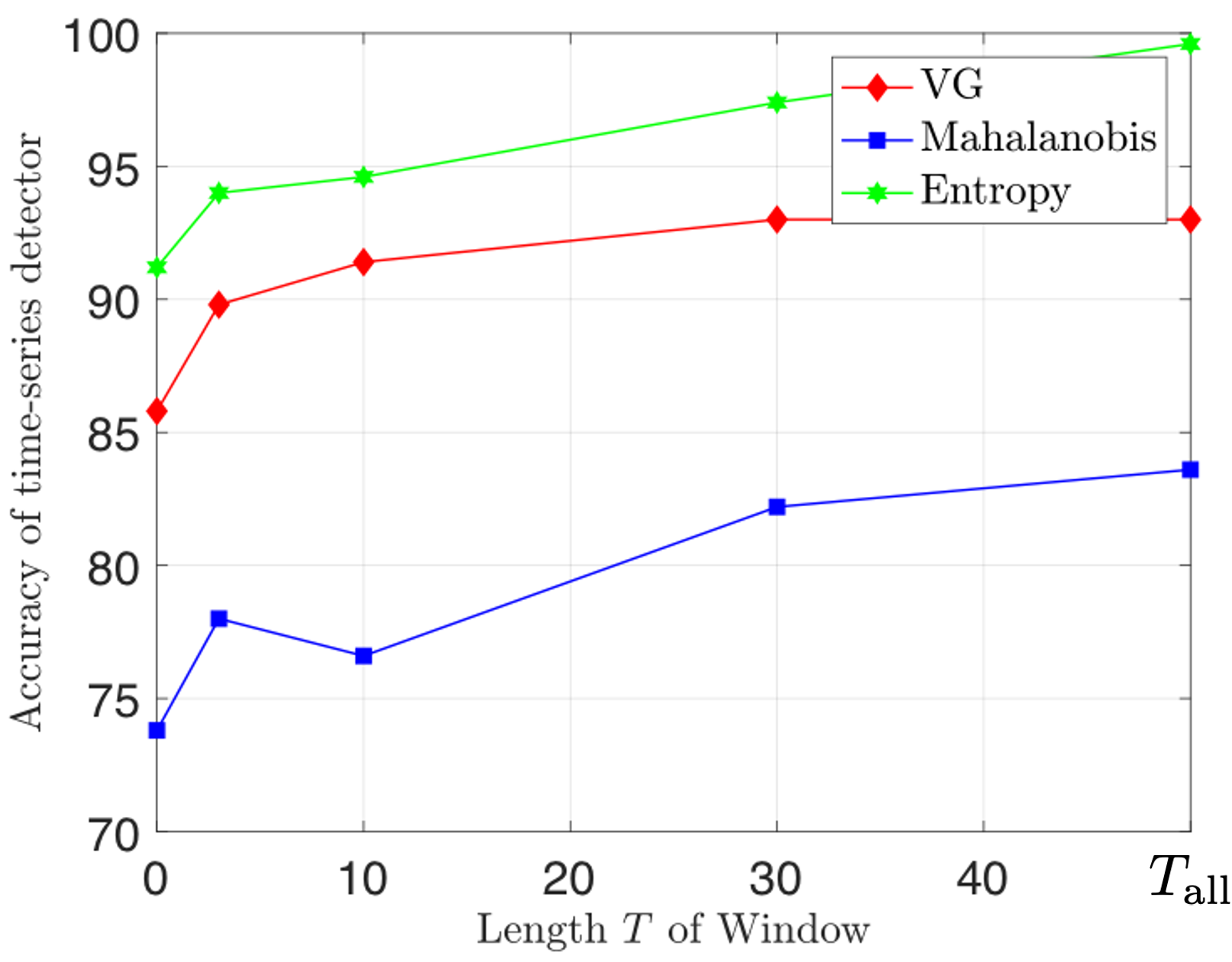}}
     \subfigure[]{
    \label{fig:advG}
    \includegraphics[width=0.45\linewidth]{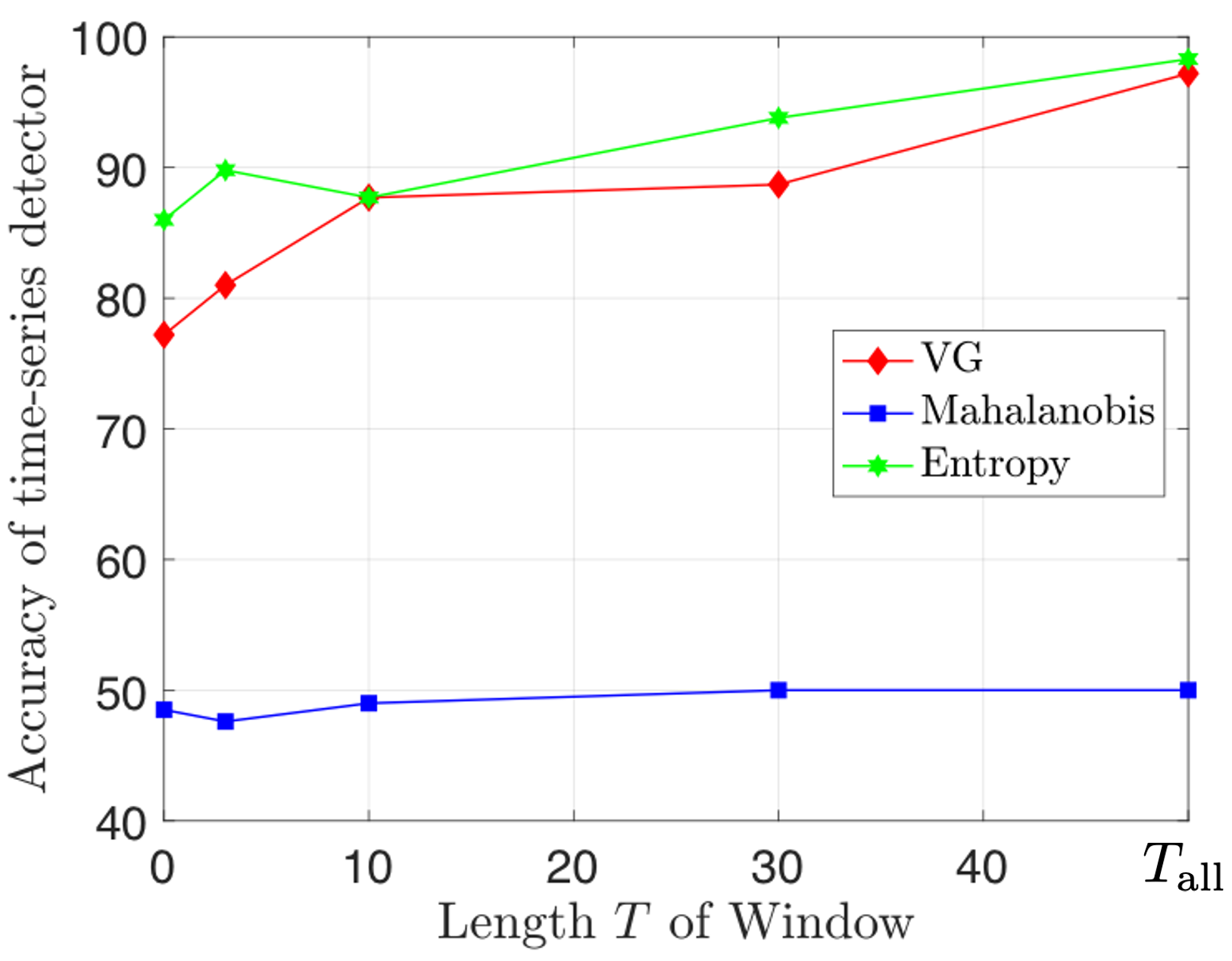}}
    \vspace{-0.2cm}
  \caption{
  Comparison of the constructed time-series detectors against clean (Fig. \ref{fig:cleanG}) and adversarial stop signs (Fig. \ref{fig:advG}) for the GTSRB CNN.
  }
  \vspace{-0.5cm}
  \label{fig:GTSRBsum}
\end{figure}

\subsection{Real-Time Computational Requirements}\label{sec:realtime}
In this section, we evaluate the real-time computational requirements of the three single-image detectors employed above and the resulting time-series detector in terms of required disk-space and execution time. MD requires storing the empirical class mean and covariance for each layer of the DNN that are used at runtime to detect adversarial inputs. Storing this information required $4.4$ MB.
%
In contrast, VG and ED do not have any disk-space requirements, as they do not rely on storing information from the training sets. Instead, they apply image transformations and filtering methods in memory. Also, VG, MD and ED require $0.018$, $0.023$ and $0.049$ secs, on average, to check if a single image is adversarial. Note that these runtimes are implementation specific and they may depend on the image dimensions. For instance, VG can be optimized by computing the softmax output of the original and the transformed image in parallel. 
The employed single-image detectors are fast enough for real time applications, such as autonomous cars. For instance, the YOLO neural network \cite{redmon2016you} typically operates in autonomous driving applications at 55 frames per second (FPS) (or at 155 FPS for a small accuracy trade off) i.e., $55$ images are generated per second \cite{yurtsever2020survey}. Thus, if the car camera operates at $55$ FPS, then a new frame/image is generated every $0.018$ seconds. In other words, VG can be used to reason about trustworthiness of each frame while the entropy detector can be used every $0.049/0.018=2.78$ frames. 
Finally, applying majority vote is quite inexpensive. For instance, majority vote required on average $7.1\times 10^{-6}$ \text{secs}, $8.1\times 10^{-6}$ \text{secs}, $7.9\times 10^{-6}$ \text{secs}, and $1.1\times 10^{-5}$ \text{secs} for $T=3, 10, 30$ and $T=T_{\text{all}}$, respectively. We note that these runtimes ignore any computational burden incurred by e.g., object tracking algorithms that will generate sequences of images containing the same object of interest.

\vspace{-0.2cm}
\section{Conclusion}
\vspace{-0.1cm}
This paper proposed an effective and modular defense against  physical adversarial attacks in time series image data. 
We provided conditions under which by using historical data, the time series detector performs better than the corresponding single image detector. Finally, we provided extensive experiments to validate the proposed time-series detector.

\vspace{-0.2cm}
\appendix
\section{Adversarial Patch: Comparative Experiments of Single-Image Detectors}\label{appB}
\vspace{-0.1cm}
In what follows, we compare the performance of the detectors VG, MD, and ED against an adversarial patch \cite{brown2017adversarial}.

\textbf{Adversarial Patch:} A targeted adversarial patch is presented in \cite{brown2017adversarial}. Particularly, to fool a DNN classifier, a part of the image is completely replaced with an adversarial patch. This patch is generated as follows. 
Given  an image $x$, a patch $p$ that may have any shape, patch location $l$, and patch transformations $t$ (e.g. rotations or scaling) a patch application operator $A(p,x,l,t)$ is defined which first applies the transformations $t$ to the patch $p$, and then applies the transformed patch to the image $x$ at location $l$. The trained patch $\hat{p}$ is trained to optimize the objective function: $\hat{p}=\mathbf{E}_{x\sim\mathcal{X},t\sim T, l\sim L}\log P(f(A(p,x,l,t))=y^*)$,
where $\mathcal{X}$ is a training set of images, $T$ is a distribution over transformations of the patch, $L$ is a distribution over locations in the image, and $P(f(A(p,x,l,t))=\hat{y})$ denotes the probability that the classifier will assign to the adversarially patched image $A(p,x,l,t)$ the target label $y^*$. Note that the expectation is over images as well, which encourages the trained patch to work regardless of what is in the background. The trained patch $\hat{p}$ is obtained in \cite{brown2017adversarial} by using a  variant of the Expectation over Transformation (EOT) framework presented in \cite{athalye2018synthesizing}. 
As shown in \cite{brown2017adversarial}, this attack is effective by either  physically placing the sticker in the vicinity of an object of interest or digitally embedding the sticker in the background of given images  (see Fig. \ref{fig:AdvPatchImageNet}).



\textbf{Dataset:} All detectors are evaluated on the ImageNet (ILSVRC2012) dataset. ImageNet contains $1.4$ million RGB $224\times 224$ images divided in $1.2$ million training images, $50,000$ validation images, $150,000$ testing images, with $1000$ classes. The adversarial patch is trained as in \cite{brown2017adversarial} to fool ResNet 50 with desired label $y^*=\text{toaster}$. Once the patch is generated, we insert it on a random location on a random ImageNet image. Some examples of adversarially attacked ImageNet images are shown in Figure \ref{fig:AdvPatchImageNet}. In our evaluation, we have used the ResNet-50 with $73.6\%$ accuracy on the ImageNet dataset \cite{he2016deep}.

\textbf{Evaluation:} To compare the performance of the detectors, we compute their corresponding ROC. VG equipped with brightness with HSV values equal to 90 and 200 yields AUROC being equal to $90.2\%$ and $93.5\%$, respectively; see Fig. \ref{fig:VG160_patch}. For example, when $HSV=160$ a threshold $\tau=3.6$ yields a TP rate equal to $90\%$ and a FP rate equal to $14.8\%$. 
The AUROC of MD  is $84.5\%$ which is lower than that of VG; see Fig. \ref{fig:MDROC_patch}. In fact, by comparing the corresponding ROC curves, we obsere that for a given FP rate, VG achieves a higher TP rate. For example, a threshold $\tau=-0.94$ yields a TP rate equal to $80\%$ and a FP rate equal to $24.9\%$. 
As in Section \ref{sec:evalDet2}, to configure ED we investigate various spatial filters. The best performance was attained using a $9\times 9$ box mean filter. In this case, ED yields  TP and FP rates that are $97\%$ and $32.5\%$, respectively. Notice that although the TP rate is quite high, the corresponding FP rate may be quite high for certain safety critical applications (e.g., autonomous driving). Overall, VG seems to have a more consistent and `robust' performance across both physical attacks considered in the paper and various models and datasets than MD and ED. We note that VG with brightness transformation and $\text{HSV}=200$ attains a high AUROC under the adversarial patch for various DNN models. For instance, for the DNN models inception v3,  vgg16,  vgg19 the AUROC is  $84.92\%$, $96.35\%$, and $95.93\%$.

\begin{figure}[t]
  \centering
  \subfigure[]{
    \label{fig:VG160_patch}
    \includegraphics[width=0.48\linewidth]{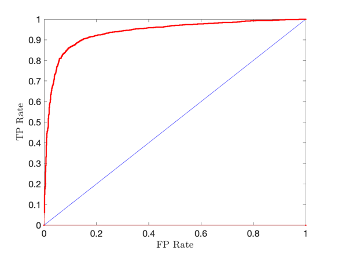}}
     \subfigure[]{
    \label{fig:MDROC_patch}
    \includegraphics[width=0.45\linewidth]{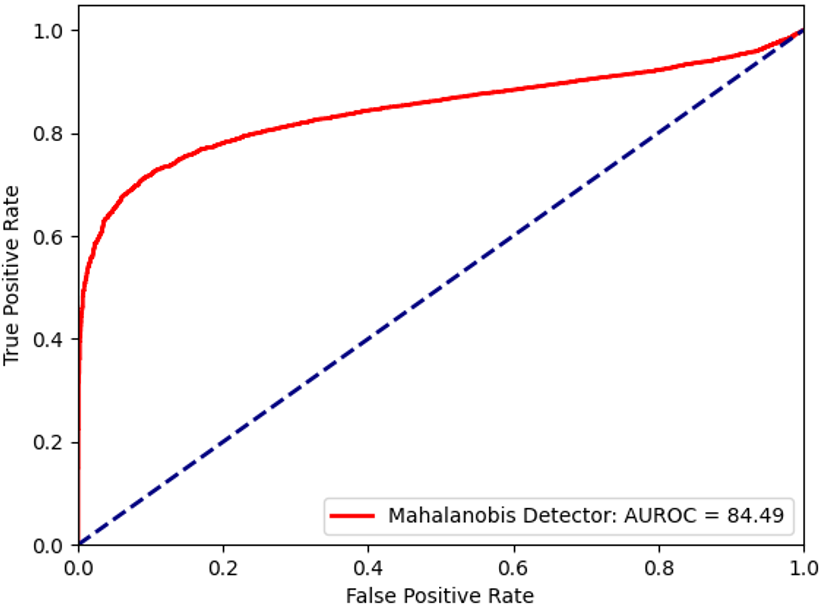}}\vspace{-0.3cm}
  \caption{
 ROC curve (red) corresponding to VG with brightness trasformation and $\text{HSV}=200$ (Fig. \ref{fig:VG160_patch}) and MD (Fig. \ref{fig:MDROC_patch}) on ImageNet and ResNet50 under the adversarial patch \cite{brown2017adversarial}.
  }
  \vspace{-0.6cm}
\end{figure}

\bibliography{YK_bib}
\bibliographystyle{IEEEtran}
\end{document}